\documentclass[runningheads]{llncs}
\usepackage[T1]{fontenc}
\usepackage{lmodern}

\usepackage{graphicx}
\usepackage{amsxtra, amsfonts, amstext, mathtools}%, , amsmath
\usepackage{booktabs}
\usepackage{nicefrac}
\usepackage{xspace}
\usepackage{url}
\usepackage{graphics,color}
\usepackage[algo2e,ruled,vlined,linesnumbered]{algorithm2e}
\usepackage{wrapfig}
\usepackage{pgfplots}
\usepackage{float}

\newcommand{\onemax}{\textsc{OneMax}\xspace}
\newcommand{\LO}{\textsc{Leading\-Ones}\xspace}
\newcommand{\leadingones}{\LO}

\newcommand{\jump}{\textsc{Jump}\xspace}

\DeclareMathOperator{\Sample}{Sample}

\newcommand{\R}{\ensuremath{\mathbb{R}}}

\newcommand{\N}{\ensuremath{\mathbb{N}}} % ohne Null!!!
\newcommand{\Z}{\ensuremath{\mathbb{Z}}}

\newcommand{\Var}{\mathrm{Var}\xspace} %use with [...]
\newcommand{\assign}{\leftarrow}
\begin{document}
{\sloppy%
\title{General Univariate Estimation-of-Distribution Algorithms}
\author{Benjamin Doerr\inst{1}%\orcidID{0000-0002-9786-220X} 
\and
Marc Dufay\inst{2}}
\authorrunning{B. Doerr, M. Dufay}
\institute{LIX, CNRS, \'Ecole Polytechnique, Institut Polytechnique de Paris, France \and \'Ecole Polytechnique, Institut Polytechnique de Paris, France}
\maketitle              % typeset the header of the contribution
\begin{abstract}
We propose a general formulation of a univariate estimation-of-distribution algorithm (EDA). It naturally incorporates the three classic univariate EDAs \emph{compact genetic algorithm}, \emph{univariate marginal distribution algorithm} and \emph{population-based incremental learning} as well as the \emph{max-min ant system} with iteration-best update. Our unified description of the existing algorithms allows a unified analysis of these; we demonstrate this by providing an analysis of genetic drift that immediately gives the existing results proven separately for the four algorithms named above. Our general model also includes EDAs that are more efficient than the existing ones and these may not be difficult to find as we demonstrate for the \onemax and \leadingones benchmarks.%, where a natural version of our algorithm turns out to be at least twice as fast as the UMDA and cGA with optimized parameters on \onemax and 20\% faster than UMDA on \leadingones.

\end{abstract}
\section{Introduction}
Estimation-of-distribution algorithms (EDAs) are a class of iterated randomized search heuristics proposed first in the 1990s~\cite{JuelsBS93}. Different from genetic algorithms (GAs), which evolve a set $P$ (``population'') of good solutions for a given problem, EDAs evolve a probability distribution (``probabilistic model'') on the set of possible solutions, hopefully in the way that good solutions have a higher probability assigned to them. Since it is clear that a set $P$ of solutions can be represented by a probability distribution (namely the uniform distribution on~$P$), EDAs (with an appropriate probabilistic model) have a much richer way of transporting information from one iteration to the next than genetic algorithms. 

Several results show that this theoretical advantage can be turned into a true advantage when running the EDA in the right way. For example, it was shown that the more cautious way of updating the probabilistic model of EDAs (as opposed to the only alternatives of a GA, which are to accept or discard a solution) can lead to a high robustness to noise~\cite{FriedrichKKS16,FriedrichKKS17}. The fact that EDAs can sample with a larger variance was shown to be advantageous for leaving local optima~\cite{HasenohrlS18,Doerr21cgajump,DoerrK21tcs,Witt21}. In~\cite{DoerrK20gecco}, it was demonstrated that the probabilistic model developed by an EDA allows to obtain much more diverse good solutions than what can be achieved by population-based algorithms.

Due to their higher simplicity, the most studied form of EDAs are \emph{univariate} ones, which sample the variables of each solution independently. When restricting ourselves to pseudo-Boolean optimization, that is, the solutions are bit-strings of length~$n$, then this means that the probabilistic model can be described by a \emph{frequency vector} $p = (p_1, \dots, p_n) \in [0,1]^n$ such that a sample $x \in \{0,1\}^n$ from this model satisfies 
\begin{equation}\label{eq:sample}
  \Pr[x_i = 1] = p_i \mbox{ independently for all $i \in [1..n] := \{1, \dots, n\}$.}
\end{equation}
The three classic univariate EDAs are \emph{population-based incremental learning (PBIL)}~\cite{Baluja94}, the \emph{univariate marginal distribution algorithm (UMDA)}~\cite{MuhlenbeinP96}, and the \emph{compact genetic algorithm (cGA)}~\cite{HarikLG99}. As observed in~\cite{KrejcaW20bookchapter}, the \emph{max-min ant system (MMAS)}~\cite{StutzleH00} with iteration-best pheromone update also is a univariate EDA (when used for pseudo-Boolean optimization). We note that the UMDA and this MMAS are special cases of PBIL. Unfortunately, with very few results existing for the PBIL, this connection so far could not be exploited extensively.

So far, these four algorithms have mostly been discussed separately, and for many aspects, only one or two of the four algorithms have been regarded. For example, there are only two mathematical analysis on how EDAs cope with Gaussian noise and these regards only the cGA~\cite{FriedrichKKS17} and the MMAS~\cite{FriedrichKKS16}. For the question how EDAs cope with local optima, the existing runtime analyses only regard the cGA~\cite{HasenohrlS18,Doerr21cgajump,Witt21} and the MMAS~\cite{BenbakiBD21}. This leaves many questions unanswered. 

We also note that many arguments used in the past were specific to the particular algorithm regarded. For example, the analyses in~\cite{HasenohrlS18,Doerr21cgajump} exploit that the cGA enjoys the property that if the sample with better fitness is closer to the optimum, then the model update will reduce the expected distance of the samples from the optimum. The MMAS does not have this property and consequently, a different proof approach was necessary in~\cite{BenbakiBD21}. 

\textbf{Our results:} In this work, we try to improve this situation by proposing a simple, yet general class of EDAs that includes the four algorithms mentioned above. Our hope is that by thus distilling the common features of these algorithms, it becomes easier to find analyses that apply simultaneously to all four algorithms. We demonstrate that this is indeed possible by proving a quantitative statement on the genetic drift effect in our EDA class. This result contains as special cases the results (separately) proven in~\cite{DoerrZ20tec}.

Our second hope is that the large class of EDAs defined by our model also contains algorithms with better performance than the four known algorithms. With elementary non-rigorous arguments, we design such an EDA and show via an experimental analysis that it is at least twice as fast at the cGA and UMDA with optimized parameters on the \onemax benchmark. We note that this new algorithm is in no way more complicated than the known special cases of our general model -- it just profits from wider ranges of allowed parameters.

\section{Previous Work}

For reasons of space and since several good surveys and textbooks are available, we describe here only the works that are really close to ours. For a general introduction to EDAs and details on applications, we refer to the surveys~\cite{HauschildP11,LarranagaL02,PelikanHL15}. 

Our work, while not purely mathematical, nevertheless is regarding EDAs more from a theoretical perspective. A very recent survey on the state of the art of the theory of EDAs is~\cite{KrejcaW20bookchapter}, broader introductions to theoretical approaches in evolutionary computation include~\cite{NeumannW10,AugerD11,Jansen13,DoerrN20}. As can easily be deduced from this survey, the theoretical understanding of EDAs is far from complete and for many basic questions, e.g., the runtime on the simple \onemax benchmark, a complete answer is still missing. What can also be observed from this survey is that essentially all previous works regard only a single univariate EDA. There are few exceptions, e.g., in~\cite{SudholtW19} both the cGA and the MMAS is analyzed, but also in these cases the results for different algorithms are proven separately. 

The only previous work we are aware of that undertakes an attempt towards a unified treatment of univariate EDAs is~\cite{FriedrichKK16}. There, the framework of an $n$-Bernoulli-$\lambda$-EDA is defined. This framework is very general and includes not only our EDA model, but in fact all univariate EDAs which sample a fixed number $\lambda$ of offspring according to~\eqref{eq:sample} and then update the probabilistic model $p$ via any deterministic function $\phi$ that takes as arguments the current model and the offspring together with their fitness. Not surprisingly, in such an extremely general model it is hard to prove meaningful results, and consequently, the particular results in~\cite{FriedrichKK16} need non-trivial additional assumptions: To show that a stable EDA is not balanced, in particular the additional assumption is made that whenever the EDA optimizes a function with neutral $i$-th bit, then at all times $t$ the sampling frequency $p_i(t)$ satisfies $\Var[p_i(t+1) \mid p_i(t)] = -a p_i(t)^2 + b p_i(t) +c$ for suitable $a,b,c \in \R$ with $0 < a < 1$, see~\cite[Theorem~10]{FriedrichKK16} (this notion has been relaxed to the requirement that $\inf\{\Var[p_i(t+1) + \mathbf{1}[p_i(t) \notin [d,1-d]] \mid p_i(t)] \mid t \in \N\}> 0$ for some $d = o(1)$ in~\cite[Theorem~6.11]{Krejca19}). Similarly, the runtime analysis on the \leadingones benchmark relies on two specific assumptions how the frequencies behave during the optimization process~\cite[Theorem~12]{FriedrichKK16}. There is no doubt that also with these restrictions, the results in~\cite{FriedrichKK16} are strong and impressive, but the need for the restrictions suggests that the $n$-Bernoulli-$\lambda$-EDA model is too general to admit strong results covering the whole model (and this is where we hope that our more narrow model is more effective).

There have also been some attempts to encompass EDAs in a model even wider. One of them is by defining these algorithms as \textit{model-based search} algorithms which rely on a parameterized probabilistic model as opposed to \textit{instance-based search} algorithms which rely on a population of solutions \cite{ZlochinBMD04}. A model-based search algorithm is described by its probabilistic model and the way it updates its model and some parallels can be made between univariate EDAs and gradient-based methods. Another approach described in \cite{OllivierAAH17} is by turning existing EDAs into a continuous-time black-box optimization method using the \textit{information-geometric optimization} (IGO) method which can then be turned back into algorithms using time discretization. Existing univariate algorithms like cGA or PBIL can be retrieved using this method. However, these approaches result in a model that is too general to obtain running time results or to obtain ideas how to set the parameter of the algorithms.

\section{Univariate EDA: Classic and New}

In this section, we first describe brief{}ly the four existing algorithms mentioned in the introduction and then derive from these a general model encompassing all four. We shall write $x \sim \Sample(p)$ to denote that $x \in \{0,1\}^n$ is sampled according to the univariate model described by the frequency vector $p \in [0,1]^n$, that is, that $x$ satisfies~\eqref{eq:sample}. We assume that each call of this sampling procedure is stochastically independent from all other samplings and possibly other random decisions of the algorithm. When an algorithm optimizing a function $f$ samples $\lambda$ individuals, we denote these by $x[1], \dots, x[\lambda]$ and we denote by $\tilde x[1], \dots, \tilde x[\lambda]$ the sorting of these by decreasing (worsening) fitness~$f$, with ties broken randomly. All algorithms initialize the univariate model as $p = (\frac 12, \dots, \frac 12)$, which gives the uniform distribution on the search space $\{0,1\}^n$. In their main loop, all sample a certain number of solutions und update the model based on the fitness of the solutions. We first describe all algorithms in the basic version without artificial frequency margins, then propose our general EDA model (also without frequency margins), and finally  discuss how to include such margins.

The \emph{compact genetic algorithm (cGA)}~\cite{HarikLG99} samples only two solutions and modifies the frequency vector by a scalar multiple of the difference between the better and the worse solution, that is, $p \assign p + \frac1K (\tilde{x}[1]-\tilde{x}[2])$. Here $K$ is the only algorithm parameter called \emph{hypothetical population size}. In other words, a frequency $p_i$ does not change if the two samples agree in the $i$-th bit, and it moves by an additive term of $\frac 1K$ towards the bit value of the better solution otherwise. Usually, $K$ is taken as an even integer since this automatically keeps the frequencies in the range $[0,1]$. For other values of $K$, one would need to cap the frequencies after the update into the interval $[0,1]$.

The \emph{univariate marginal distribution algorithm (UMDA)}~\cite{MuhlenbeinP96} with parameters $\lambda, \mu \in \Z_{\ge 1}$ samples $\lambda$ solutions and updates the model to the average of the $\mu$ best solutions, that is, $p \assign \frac 1\mu \sum_{i=1}^\mu \tilde x[i]$.

The \emph{max-min ant system (MMAS)}~\cite{StutzleH00} with iteration-best update besides the sample size $\lambda$ has the \emph{learning rate} $\rho \in ]0,1]$ (\emph{pheromone evaporation rate} in the ant colony optimization language) as second parameter. Only the best offspring is used for the model update and it enters the model with weight $\rho$, that is, the model update is $p \assign (1-\rho) p + \rho \tilde x[1]$. 

\emph{Population-based incremental learning (PBIL)}~\cite{Baluja94} selects $\mu$ out of $\lambda$ solutions and combines their average weighted by $\rho$ with the current model: $p \assign (1-\rho) p + \rho \frac 1\mu \sum_{i=1}^\mu \tilde x[i]$. Consequently, PBIL has as special cases both the UMDA (by taking $\rho=1$) and the MMAS (by taking $\mu=1$).  

The pseudocodes for these four algorithms are given in Algorithms~\ref{alg:cga} to~\ref{alg:pbil}. As can easily be seen, in all four cases the new model is a linear combination of the samples and the old model. This suggests the following \emph{general univariate EDA model}. Let $\lambda \in \Z_{\ge 1}$ the sample size and $\gamma_0, \gamma_1, \dots, \gamma_\lambda \in \R$ such that $\sum_{i=0}^\lambda \gamma_i = 1$. The general univariate EDA in its main loop samples $\lambda$ solutions and updates the frequency vector to $p \assign \gamma_0 p + \sum_{i=1}^\lambda \gamma_i \tilde x[i]$, where this is to be understood that frequencies below zero or above one are replaced by zero or one. The complete pseudocode is given in Algorithm~\ref{alg:psm}.

\begin{algorithm2e}[ht]%
	$p(0) = \left(\frac{1}{2}, \dots, \frac{1}{2}\right) \in [0,1]^n$ \\
	\For{$t = 1,2, \ldots$}{
	$x[1] \sim \Sample(p(t-1))$  \\
	$x[2] \sim \Sample(p(t-1))$ \\
	\If{$f(x[1]) \geq f(x[2])$} {
		$p(t) = p(t-1) + \frac{1}{K}( x[1] - x[2])$
	} \Else {
		$p(t) = p(t-1) + \frac{1}{K}( x[2] - x[1])$
	}
	$p(t) = \max(0, \min(1, p(t)))$
}	
\caption{The cGA with parameter $K > 0$, maximizing a given function $f : \{0,1\}^n \to \R$.}
\label{alg:cga}
\end{algorithm2e}

\begin{algorithm2e}[ht]%
	$p(0) = \left(\frac{1}{2}, \dots, \frac{1}{2}\right) \in [0,1]^n$ \\
	\For{$t = 1,2, \ldots$}{
	\For{$i = 1,2, \ldots, \lambda$}{
		$x[i] \sim \Sample(p(t-1))$ 
	}
	Sort the individuals into $\tilde{x}[1], \ldots, \tilde{x}[\lambda]$ ordered by worsening fitness\\
	\emph{\%\% Update the frequency} \\
		$p(t) = \frac{1}{\mu} \sum_{i=1}^\mu  \tilde{x}[i]$ 
	}
\caption{The UMDA with parameters  $\lambda \in \Z_{\ge 1}$ and $\mu \in \ [1..\lambda]$.}
\label{alg:umda}
\end{algorithm2e}

\begin{algorithm2e}[ht]%
	$p(0) = \left(\frac{1}{2}, \dots, \frac{1}{2}\right) \in [0,1]^n$ \\
	\For{$t = 1,2, \ldots$}{
	\For{$i = 1,2, \ldots, \lambda$}{
		$x[i] \sim \Sample(p(t-1))$ 
	}
	Find an individual with the best fitness $\tilde{x}[1]$ \\
	\emph{\%\% Update the frequency} \\
		$p(t) = (1 - \rho) p(t-1) + \rho \tilde{x}[1]$ 
	}
\caption{The MMAS with parameters  $\lambda \in \Z_{\ge 1}$ and evaporation factor $\rho \in \ ]0,1]$.}
\label{alg:mmas}
\end{algorithm2e}

\begin{algorithm2e}[ht]%
	$p(0) = \left(\frac{1}{2}, \dots, \frac{1}{2}\right) \in [0,1]^n$ \\
	\For{$t = 1,2, \ldots$}{
	\For{$i = 1,2, \ldots, \lambda$}{
		$x[i] \sim \Sample(p(t-1))$ 
	}
	Sort the individuals into $\tilde{x}[1], \ldots, \tilde{x}[\lambda]$ ordered by their fitness\\
	\emph{\%\% Update the frequency} \\
		$p(t) = (1 - \rho) p(t-1) + \frac{\rho}{\mu} \sum_{i=1}^\mu  \tilde{x}[i]$ 
	
	}
\caption{PBIL with parameters $\rho \in \ ]0,1]$, $\lambda \in \N$ and $\mu \in \ [1..\lambda]$.}
\label{alg:pbil}
\end{algorithm2e}

We immediately see that the general univariate EDA contains the four algorithms above as special cases. We obtain the cGA by taking $\lambda = 2$, $\gamma_0 = 1$, $\gamma_1 = \frac 1K$, and $\gamma_2 = -\frac 1K$. For the UMDA with parameters $\lambda$ and $\mu$, we use the same $\lambda$ and the weights $\gamma_0 = 0$, $\gamma_1 = \dots = \gamma_\mu = \frac 1\mu$ and $\gamma_{\mu+1} = \dots = \gamma_{\lambda} = 0$. The MMAS results from taking $\gamma_0 = 1-\rho$, $\gamma_1 = \rho$, and $\gamma_{2} = \dots = \gamma_{\lambda} = 0$. Finally, PBIL is the general EDA with $\gamma_0 = 1-\rho$, $\gamma_1 = \dots = \gamma_\mu = \frac \rho\mu$, and $\gamma_{\mu+1} = \dots = \gamma_{\lambda} = 0$.

\begin{algorithm2e}%
	$p(0) = \left(\frac{1}{2}, \dots, \frac{1}{2}\right) \in [0,1]^n$ \\
	\For{$t = 1,2, \ldots$}{
	\emph{\%\%Sample the individuals} \\
	\For{$i = 1,2, \ldots, \lambda$}{
		\emph{\%\%Generate the i-th individual $x[i]$} \\
		$x_t[i] \sim \Sample(p(t-1))$ 
	}
	Sort the individuals into $\tilde{x}_t[1], \ldots, \tilde{x}_t[\lambda]$ by worsening fitness\\
	\emph{\%\% Update the frequency} \\
		$p(t) = \max(0, \min(1, \gamma_0 p(t-1) + \sum_{i=1}^\lambda \gamma_i \tilde{x}[i]))$ 
	
	}
\caption{Our general EDA algorithm defined by $(\gamma_i)_{i = 0,\ldots, n}$ such that $\sum_{i=0}^\lambda \gamma_i = 1$.}
\label{alg:psm}
\end{algorithm2e}

\section{Genetic Drift}

Genetic drift is the phenomenon that the sampling frequencies of the probabilistic model move in some direction not because of the feedback from the fitness, but by an unfortunate accumulation of the small random movements that occur when there is no clear signal from the fitness. Genetic drift is problematic in that it can move frequencies close to the boundary values $0$ and $1$, where they tend to stay longer. This phenomenon and its drawbacks were first discussed in the series of works~\cite{Shapiro02,Shapiro05,Shapiro06}. After a long sequence of fundamental results such as~\cite{DangLN19,Droste06,FriedrichKK16,KrejcaW20,LenglerSW21,SudholtW19,Witt19,DoerrZ20tec}, mostly runtime analyses which only apply to a regime with low genetic drift, we now understand this phenomenon quite well. For reasons of completeness, we note that EDAs can also be successful in regimes with genetic drift, see, e.g., the runtimes results~\cite{DangLN19,Witt19} for the UMDA on \onemax and \leadingones when the population size is logarithmic, but the general understanding is that genetic drift is dangerous and examples like the analyses of the UMDA on the DLB problem~\cite{LehreN19foga,DoerrK21ecj} show that genetic drift can lead to drastic performance losses.

The tightest quantitative statements on genetic drift were given in~\cite{DoerrZ20tec}. They were proven via separate analyses for the cGA and PBIL (which imply the corresponding results for the UMDA and MMAS). With our general model for univariate EDAs, we can now provide a unified analysis for these classic algorithms (and all algorithms that will be defined in the future that fit into this model). 

Genetic drift is usually studied by regarding a neutral bit, that is, a bit that has no influence on the fitness (note that such results imply similar results for bits that are neutral only for a certain time as in the \leadingones benchmark or bits that have a preference for one value as in monotonic functions, see~\cite{DoerrZ20tec}). By symmetry, the expected value of the sampling frequency of a neutral bit is always $\frac 12$ (and in fact, the distribution of this frequency is also symmetric around $\frac 12$). Nevertheless, as discussed above, the random fluctuations stemming from the updates of the probabilistic model will move this frequency towards the boundary values $0$ and $1$, and this is the phenomenon of genetic drift. Genetic drift can be quantified, e.g., via statements on the first time that the frequency leaves some middle ground, e.g., the interval $[\frac 13,\frac 23]$. 

In the remainder of this section, let us assume that the first bit of our objective function $f$ is neutral. Then this bit has no influence on the selection, and consequently for all $i \in [1 .. \lambda]$, we have $\tilde{x}_1[i] \sim \mathcal{B}\left(p_1(t-1)\right)$.
%\begin{equation*}
 %\tilde{x}_1[i] \sim \mathcal{B}\left(p_1(t-1)\right).
%\end{equation*}
For simplicity, we write $x^i_t = \tilde{x}_1[i], p_t = p_1(t)$
%\begin{equation*}
%x^i_t = \tilde{x}_1[i], p_t = p_1(t)
%\end{equation*}
for all $t \geq 0, i \in [1.. \lambda]$. We will also assume that we are not in a totally degenerate case, so there exists $i \in [1.. \lambda]$ such that $\gamma_i \ne 0$. %The proof of the following lemmas is in the appendix.

\begin{lemma}\label{lemma:martingale}
The sequence $\big(\frac{p_t (1 - p_t)}{( 1 - \sum_{i=1}^\lambda \gamma_i^2)^t}\big)_{t \geq 0}$ with respect to the filtration $(p_t)_{t \geq 0}$ is a martingale.
\end{lemma}

We note that this result is quite beautiful because it gives a good insight on the behavior of a neutral bit and no approximation was needed, allowing us to obtain a martingale and not a supermartingale or a submartingale like what is usually the case. 
\ifdefined\arxiv
  The proof of this and all other results can be found in the appendix.
\else 
For reasons of space, the formal proof of this and the other results of this paper had to be omitted in the conference version~\cite{DoerrD22ppsn}. They can be found in the appendix of this preprint.
\fi

Using this result, we can find an upper bound on the expected time for a neutral bit frequency to move away from $1/2$.

\begin{lemma}\label{lemma:hitting}
Let $T_L = \min \{ t \geq 0, p_t \leq 1/3 \ \text{or} \ p_t \geq 2/3\}$ be the first time for a neutral bit to leave $[1/3, 2/3]$. Then $E[T_L] = \mathcal{O}\left( \frac{1}{\sum_{i=1}^\lambda \gamma_i^2} \right)$.
%\begin{align*}
%E[T_L] = \mathcal{O}\left( \frac{1}{\sum_{i=1}^\lambda \gamma_i^2} \right).
%\end{align*}
\end{lemma}

To obtain a lower bound and more precise concentration results, we can use a Hoeffding inequality in a way similar, but more general than what was done in~\cite{DoerrZ20tec}.

\begin{lemma}\label{lemma:concentration}
For all $T \in \N$ and $\delta > 0$, we have
\begin{align*}
P\left[ \forall t \in [0..T], |p_t - 1/2| < \delta \right] \geq 1 - 2\exp\left( \frac{- \delta^2}{2 T \sum_{i=1}^\lambda \gamma_i^2} \right).
\end{align*}
\end{lemma}

%Using this lemma, we get the following, for $\displaystyle T_0 = \frac{\left(\sum_{i=1}^\lambda \gamma_i^2 \right)^{-1}}{4 \cdot 36 \log n}$.
%\begin{align*}
%P\left[ \forall t \in [0..T_0], p_t \in [1/3, 2/3] \right] 
%& \geq P\left[ \forall t \in [0..T_0], |p_t - 1/2| < 1/6 \right] \\
%& \geq 1 - 2\exp\left( \frac{- (1/6)^2}{2 T_0 \sum_{i=1}^\lambda \gamma_i^2} \right)  = 1 - \frac{2}{n^2}.
%\end{align*}
%
%
With $T_0 = \frac{\left(\sum_{i=1}^\lambda \gamma_i^2 \right)^{-1}}{4 \cdot 36 \log n}$ and a union bound, we obtain the following guarantee that neutral frequencies stay away from the boundaries.
\begin{corollary}
Assuming that all bits are independent and neutral, with high probability, before iteration $T_0$, all bits frequencies stay within the range $[1/3, 2/3]$.
\end{corollary}
%\begin{proof}
%We use a union bound on all bits
%\begin{align*}
%P&\left[\exists i \in [1..n], t \in [0..T_0], p_i(t) \notin [1/3, 2/3] \right] \\ 
%& \leq n P\left[ \exists t \in [0..T_0], p_t \notin [1/3, 2/3] \right] \\
%& \leq  n \frac{2}{n^2} = o(1). 
%\end{align*}
%\end{proof}

As in~\cite[part VI]{DoerrZ20tec}, this result can be extended to bits with a preference. 
For a fitness function $f$, we say that it is \textit{weakly preferring} $1$ in bit $i$ if for all $(x_1, \dots, x_{i-1}, x_{i+1},\dots, x_n) \in \{0,1\}^{n-1}$ we have
\begin{align*}
f(x_1, \dots, x_{i-1},1, x_{i+1},\dots, x_n) \geq f(x_1, \dots, x_{i-1},0, x_{i+1},\dots, x_n).
\end{align*}
Many common fitness functions like \onemax or \leadingones are \textit{weakly preferring} $1$ in any bit.

\begin{corollary}
 If the fitness function is \textit{weakly preferring} a $1$ on all of its bits, then we have %with high probability a lower bound for the frequency of each bit before $T_0$
$P\left[\forall i \in [1..n], \forall t \in [0..T_0], p_i^t \geq 1/3 \right]  = 1 - o(1)$.
%\begin{align*}
%P\left[\forall i \in [1..n], \forall t \in [0..T_0], p_i^t \geq 1/3 \right] & = 1 - o(1).
%\end{align*}
\end{corollary}

\section{Optimizing the $(\gamma_i)_i$}\label{part:gamma_rep}

A second advantage of our general formulation of univariate EDAs, besides giving unified proofs, could be that this broad class of algorithms contains EDAs that are superior to the four special cases that have been regarded in the past. To help finding such algorithms, we now discuss the influence on the $\gamma_i$ on the optimization progress. Since different $\gamma_i$ might be profitable in different stages of the optimization progress, we analyze their effect in a single iteration, that is, we condition on the current frequency vector. To ease the notation, let us call this frequency vector $p$ (without any time index). Let $\tilde x[1], \dots, \tilde x[\lambda]$ denote the $\lambda$ samples taking in this iteration, sorted already by decreasing fitness. Then, ignoring the influence of frequency boundaries, the next frequency vector $p'$ satisfies $p' = \gamma_0 p + \sum_{i=1}^\lambda \gamma_i \tilde x[i]$.

We would like to have an idea of what the optimal $(\gamma_i)$ with respect to minimizing the expected convergence time to reach the optimal solution would look like. To do so, we look during a single iteration for the \onemax function at the best distribution of $(\gamma_i)$ while keeping the genetic drift minimal. During iteration $t$, let $X(t)$ be a random variable following distribution $(p_i(t))_i$, we want to maximize $E[f(X(t+1))]$ knowing the previous distribution. \onemax being linear, using the linearity of expectation on all the different bits, we have
\begin{align*}
E[f(X(t+1))] &= \gamma_0 E[f(X(t))] + \sum_{i=1}^\lambda \gamma_i E[f(\tilde{x}[i])] \\
&= \left(1 - \sum_{i=1}^\lambda \gamma_i \right) E[f(X(t))] + \sum_{i=1}^\lambda \gamma_i E[f(\tilde{x}[i])] \\
& = E[f(X(t))] + \sum_{i=1}^\lambda \gamma_i \left(E[f(\tilde{x}[i])) - E[f(X(t))] \right).
\end{align*}
Let us assume that $(\tilde{\gamma}_i)_i$ are optimal for the current iteration and let $\delta = \sum_{i=1}^\lambda \tilde{\gamma}_i^2$ be the genetic drift. Because this iteration maximizes the expected outcome of the next distribution while minimizing the genetic drift, it is a solution to
\begin{align*}
\text{Maximize: } & E[f(X(t))] + \sum_{i=1}^\lambda \gamma_i \left(E[f(\tilde{x}[i])] - E[f(X(t))]\right) \\
\text{Subject to: } & \sum_{i=1}^\lambda \gamma_i^2 \leq \delta
\end{align*}
Both the function to optimize and the constraint are polynomial so differentiable. Moreover the set solution to the constraint is bounded and closed, so it is compact. Therefore an optimal solution exists and we can use the method of Lagrange multipliers to find it: there exists a Lagrange multiplier $\alpha \leq 0$ such that
\begin{align*}
\begin{bmatrix}
E[f(\tilde{x}[1])] - E[f(X(t))] \\
E[f(\tilde{x}[2])] - E[f(X(t))] \\
\dots \\
E[f(\tilde{x}[\lambda])] - E[f(X(t))] \\
\end{bmatrix} + \alpha \begin{bmatrix}
2 \tilde{\gamma}_1 \\
2 \tilde{\gamma}_2 \\
\dots \\
2 \tilde{\gamma}_\lambda
\end{bmatrix} = 0.
\end{align*}
So $(\tilde{\gamma}_i)_i$ are proportional to $\left(E[f(\tilde{x}[i])] - E[f(X(t))]\right)_i$. Because $(\tilde{x}[i])$ are sorted according to their fitness, $\left(E[f(\tilde{x}[i])]\right)_i$ is decreasing so $(\tilde{\gamma}_i)_i$ should also be decreasing.

\section{Designing New Univariate EDAs}

In this section, we propose two new univariate EDAs (that is, EDAs within our framework with $\gamma_i$ that do not lead to one of the four classical algorithms) and analyze them via experimental means. Given the momentary state of the art in mathematical runtime analysis of EDAs, it seems out of reach to conduct a mathematical runtime analysis precise enough to make visible the influence of the $\gamma_i$ on the runtime. The main insight derived from this part of our work is that with not much effort, one can find univariate EDAs which outperform the classic univariate EDAs. We conduct this line of research for the two classic benchmarks \onemax and \leadingones.

\textbf{\onemax:}
Since univariate EDAs sample the bits independently and since in the \onemax benchmark each bit contributes the same to the fitness, we expect a somewhat regular behavior in a set of independent samples: Those with best fitness will have many bits set correctly, those with lowest fitness with miss many bit values. This, together with the considerations of the previous section, suggests to give more weights to better samples in the frequency update, and to do this in a somewhat continuous manner. One way of doing so is taking 
\begin{align}
\gamma_0 = 1 - \beta \sum_{i = 1}^\lambda (1 - \tfrac{i}{\lambda / 2}) \approx 1 \mbox{ and }
\gamma_i = \beta (1 - \tfrac{i}{\lambda / 2}) \mbox{ for $i \in [1 .. \lambda]$},\label{eq:gamma}
\end{align}
%\begin{align}
%&\gamma_0 = 1 - \beta \sum_{i = 1}^\lambda (1 - \tfrac{i}{\lambda / 2}) \approx 1 \mbox{ and}\nonumber\\ 
%&\gamma_i = \beta (1 - \tfrac{i}{\lambda / 2}) \mbox{ for $i \in [1 .. \lambda]$},\label{eq:gamma}
%\end{align}
where $\beta$ is a positive number still to be determined. While not perfectly symmetric, essentially here $\tilde x[i]$ and $\tilde x[\lambda-i]$ have weights of opposite sign, hence $\gamma_0$ is essentially one. 

We compare this new EDA with the two classic ones UMDA and cGA with optimized parameters. We do not regard the other two classic EDAs since with their learning rate $\rho$ they are structurally quite different and it is less understood what are good parameter settings for these. We note that there is no indication in the literature that the MMAS or PBIL with their slightly cautious learning mechanism could outperform the other two algorithms on a simple unimodal benchmark such as \onemax.

For the UMDA and cGA, we determine good parameter values as follows. For the UMDA, we chose to fix $\lambda$ as $\lfloor \log n \sqrt{n} \rfloor$ since both theoretical and experimental results show that this leads to good performances \cite{Witt19}. We use the same value of $\lambda$ for our EDA. Still for the UMDA, we set $\mu = \lfloor \lambda / 3 \rfloor$ as this gave the best expected runtimes in the experiments we conducted to opitmize the parameters of the UMDA. For cGA, the only parameters that needs to be determined is the hypthetical population size~$K$. From~\cite[Figure~1]{DoerrZ20gecco}, we know that the expected runtime of the cGA on \onemax is roughly a unimodal function in~$K$.\footnote{We know that~\cite{LenglerSW21} proved that the runtime of the cGA on \onemax is not unimodal in $K$ when $n$ is large enough, but apparently this asymptotic results becomes relevant only for very large population sizes.} Since $\beta$ in our algorithm plays a similar role as $K$ in the cGA (namely it regulates the strength of the model update), we expect a similar unimodal dependence on $\beta$ for our algorithms, which we confirm in experiments. For that reason, for each problem size $n$ we determined the optimal values for $K$ and $\beta$ via ternary search.

Figure~\ref{fig:onemax} displays the average (in 200 runs) runtime of these three algorithms for different problems sizes. These results show that our general algorithm with a gamma distribution that was not used in previous algorithms is about twice as fast as the optimized UMDA and cGA. This suggest that it is not too difficult to find in our broad class of univariate EDAs new algorithms which are significantly faster than the classic algorithms.

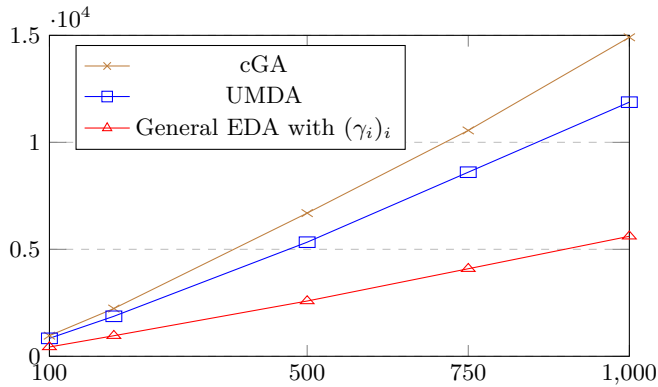
\begin{figure}[t]\centering
\begin{tikzpicture}
\begin{axis}[
    xmin=100, xmax=1000,
    ymin=0, ymax=15000,
    xtick={100,500,750,1000},
    ytick={0,5000, 10000, 15000},
    legend pos=north west,
    ymajorgrids=true,
    grid style=dashed,
    xscale=1.5,
		scale=0.75
]

\addplot[
    color=brown,
    mark=x,
    ]
    coordinates {
    (100, 964) (200, 2232) (500, 6692) (750, 10556) (1000, 14908)
    };
    \addlegendentry{cGA}

\addplot[
    color=blue,
    mark=square,
    ]
    coordinates {
    (100, 839) (200, 1870) (500, 5333) (750, 8604) (1000, 11877)
    };
    \addlegendentry{UMDA}
    
\addplot[
    color=red,
    mark=triangle,
    ]
    coordinates {
    (100, 435) (200, 961) (500, 2583) (750, 4092) (1000, 5609)
    };
    \addlegendentry{General EDA with $(\gamma_i)_i$}
    
\end{axis}
\end{tikzpicture}
\caption{Average running times (in fitness evaluations) of cGA (with optimized value of $K$), UMDA (with fixed $\lambda = \lfloor \log n \sqrt{n} \rfloor$ and optimized value $\mu = \lambda/3$), and our general algorithm with fixed gamma as in~\eqref{eq:gamma} and $\beta$ optimized, on the \onemax benchmark with problem size between $n=100$ and $n=1000$.}
\label{fig:onemax}
\end{figure}

\textbf{\leadingones:}
We undertook a similar work for the \leadingones benchmark. In this function, the bits do not contribute independently to the fitness, so our considerations valid in the design of the EDA above are not valid anymore. More detailedly, search points with low fitness reveal very little information how good solutions look like. For this reason, we design our new EDA in the way that such solutions are not taken into account for the model update. Without any optimizing, we set the cutoff for this regime at $\lambda/3$, that is, we have $\tilde \gamma_i = 0$ for all $i > \lambda / 3$. For the remaining samples, we expect some positive information towards the optimum, and again we expect this to be stronger for better solutions, so we take $\tilde \gamma_i$ proportional to $\lfloor \lambda / 3 \rfloor - (i-1)$. With no particular reason, we decided to define an EDA resembling the UMDA, that is, we take $\tilde \gamma_0 = 0$ and 
\begin{align}
\tilde{\gamma}_i = \frac{\lfloor \lambda / 3 \rfloor - (i-1)}{\sum_{j=1}^{\lfloor \lambda / 3 \rfloor}  \lfloor \lambda / 3 \rfloor - (j-1)}
\end{align}
for all $i \in [1..\lambda/3]$.
%\begin{align}
%\tilde{\gamma}_i = \begin{cases}
%\frac{\lfloor \lambda / 3 \rfloor - (i-1)}{\sum_{j=1}^{\lfloor \lambda / 3 \rfloor}  \lfloor \lambda / 3 \rfloor - (j-1)} \  & \text{if $i \leq \lfloor \lambda / 3 \rfloor$}\\
%0 & \text{otherwise}
%\end{cases}
%\end{align}

In Figure~\ref{fig:leadingones}, we experimentally compare the EDA just designed, the EDA designed in the previous subsection, and the UMDA with parameters optimized (for \leadingones) as described in the previous subsection. As expected, the running time of our general algorithm with the $(\gamma_i)_i$ chosen in the previous subsection is not very good (roughly by 25\% worse that the UMDA). The EDA just designed, however, beats the UMDA with optimized parameters by roughly 20\%. This again shows that with moderately effort, one can find superior EDAs in the class of univariate EDAs defined in this work. 

We admit that the \onemax and \leadingones benchmarks are well-understood, so designing a better univariate EDA for a complicated real-world problem will require more work. Nevertheless, we are optimistic that using intuitive ideas such as the ones above, e.g., a continuous dependence of the $\gamma_i$ on the rank~$i$, together with some trial-and-error experimentation can lead to good EDAs (better than the classic ones) also for more complex problems.
%\iffalse
%\begin{figure}[h]\centering
%\begin{tikzpicture}
%\begin{axis}[
    %xlabel={n},
    %ylabel={Expected number of iterations},
    %xmin=100, xmax=1500,
    %ymin=0, ymax=80,
    %xtick={100,500,750,1000,1250,1500},
    %ytick={0,20, 40, 60,80},
    %legend pos=north west,
    %ymajorgrids=true,
    %grid style=dashed,
    %xscale=1.5
%]
%
%\addplot[
    %color=blue,
    %mark=square,
    %]
    %coordinates {
    %(100,18.26) (200,25.3) (500, 38.65) (750, 47.54) (1000, 54.38) (1250, 61.13) (1500, 66.67)
    %};
    %\addlegendentry{UMDA}
%
%\addplot[
    %color=red,
    %mark=triangle,
    %]
    %coordinates {
    %(100,9.47) (200,12.98) (500, 18.72) (750, 22.61) (1000, 25.68) (1250, 28.35) (1500, 29.38)
    %};
    %\addlegendentry{General algorithm}
    %
%\end{axis}
%\end{tikzpicture}
%\caption{Experimental running time on \onemax of UMDA and our general algorithm with a "diagonal" gamma distribution}
%\end{figure}
%\fi
\begin{figure}[t]\centering
\begin{tikzpicture}
\begin{axis}[
    xmin=50, xmax=500,
    ymin=0, ymax=410000,
    xtick={50, 100,200,300,400, 500},
    ytick={0,100000, 200000, 300000, 400000},
    legend pos=north west,
    ymajorgrids=true,
    grid style=dashed,
    xscale=1.5,
		scale=0.75
]

\addplot[
    color=red,
    mark=triangle,
    ]
    coordinates {
    (50, 3500) (100, 15088) (200, 60088) (300, 139160) (400, 247320) (500, 403100)
    };
    \addlegendentry{General EDA with $(\gamma_i)_i$}
    
\addplot[
    color=blue,
    mark=square,
    ]
    coordinates {
    (50, 2940) (100, 12696) (200, 51060) (300, 118580) (400, 204840) (500, 319839)
    };
    \addlegendentry{UMDA}

\addplot[
    color=purple,
    mark=o,
    ]
    coordinates {
    (50, 2688) (100, 10212) (200, 39442) (300, 89964) (400, 158760) (500, 248532)
    };
    \addlegendentry{General EDA with $(\tilde{\gamma}_i)_i$}
    
\end{axis}
\end{tikzpicture}
\caption{Average running times (in fitness evaluations) over 200 runs of the classic UMDA (with optimized parameters) and the two EDAs designed in this section, on \leadingones with problem size between $n=50$ and $n=500$. The $\tilde\gamma_i$ chosen with consideration of elementary properties of \leadingones clearly outperform the other two algorithms.}
\label{fig:leadingones}
\end{figure}
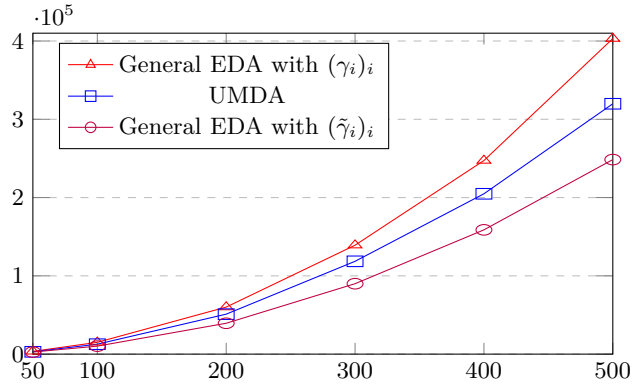

\section{Conclusion}

In this work, we proposed a general formulation of a univariate EDA. It captures the three main univariate EDAs and the MMAS ant colony optimizer with iteration-best update. Our formulation allows to phrase proofs, so far conducted individually for the different algorithms, in a unified manner. We demonstrate this for a recent quantitative analysis of genetic drift. We are optimistic that our formulation also allows to conduct some of the existing runtime analyses in a unified manner. This would be particularly interesting as here many results have been shown only for some of the classic algorithms, e.g., the runtime analyses on the \onemax and \jump benchmarks as well as the results on noisy optimization. However, given the high complexity of the existing analyses for particular algorithms, this might be a challenging task.

Our general formulation also allows to define new univariate EDAs, which might turn out to be superior to the existing ones. With intuitive arguments, we define such EDAs and show experimentally that they beat existing EDAs for the \onemax and \leadingones benchmarks. 
%The weights $(\gamma_i)$ we used to define this algorithm are natural and fit to our theoretical understanding: We give the best and worst solutions a high (positive or negative) weight. Solutions in the middle range are still used for the model update, but with less weight since on the one hand, we expect such solutions to give a weaker signal towards better solutions, but on the other hand, their weights contribute to the genetic drift in the same way as all other weights. From these considerations, 
We are optimistic that this approach can be profitable also for other optimization problems.

\subsection*{Acknowledgment}

This work was supported by a public grant as part of the
Investissements d'avenir project, reference ANR-11-LABX-0056-LMH,
LabEx LMH.

%\clearpage
%\bibliographystyle{alpha}
%\bibliographystyle{splncs04}
%\bibliography{ich_master,alles_ea_master}

%%

%comment out the following two lines for the version with appendix
%}%end sloppy
%\end{document}

\newpage
\appendix
\section{Appendix}

This appendix contains the mathematical proofs which had to be omitted in the main paper for reasons of space. They are given here only in case a reviewer wants to access them. The paper without this appendix is our submission to PPSN and we feel that it is perfectly accessible without this appendix. Nevertheless, upon acceptance, we will make a version including this appendix available on the arxiv preprint server.

\subsection{Proof of Lemma \ref{lemma:martingale}}
\begin{proof}
It suffices to show that for all $t \geq 0$, we have
\begin{equation}
E[p_{t+1}(1-p_{t+1})\mid p_t] = \left( 1 - \sum_{i=1}^\lambda \gamma_i^2 \right) p_t(1-p_t).
\end{equation}
Let $t \geq 0$. Using the fact that $(p_t)_{t \geq 0}$ is a martingale, we have
\begin{align*}
E[p_{t+1}(1-p_{t+1})\mid p_t] 
& = E[p_{t+1}\mid p_t] - E[p_{t+1}^2\mid p_t] \\
& = E[p_{t+1}\mid p_t] - \left( E[p_{t+1}^2\mid p_t] - E[p_{t+1}^2\mid p_t]^2 + E[p_{t+1}^2\mid p_t]^2 \right) \\
& = E[p_{t+1}\mid p_t](1 - E[p_{t+1}\mid p_t]) - \Var (p_{t+1}\mid p_t) \\
& = p_t( 1 - p_t ) - \Var (p_{t+1}\mid p_t).
\end{align*}
Moreover, conditionally to the value of $p_t$, the $(x_i^t)_{i \in [1 .. \lambda]}$ are independent Bernoulli random variable, each of parameter $p_t$. Therefore
\begin{align*}
 \Var (p_{t+1}\mid p_t) 
 & = \Var \left(\gamma_0 p_t + \sum_{i=1}^\lambda \gamma_i x_i^t \mid p_t\right) \\
 & = \gamma_0^2 \Var(p_t \mid p_t) + \sum_{i=1}^\lambda \gamma_i^2 \Var (x_i^t \mid p_t) \\
 & = 0 + \sum_{i=1}^\lambda \gamma_i^2 p_t (1 - p_t).
\end{align*}
Using the two previous equations, we obtain the desired result
\begin{align*}
E[p_{t+1}(1-p_{t+1})\mid p_t] = \left( 1 - \sum_{i=1}^\lambda \gamma_i^2 \right) p_t(1-p_t).
\end{align*}
\end{proof}
\subsection{Proof of Lemma \ref{lemma:hitting}}
\begin{proof}
For $t \geq 0$, let
\begin{align*}
d_t = \min( p_t, 1-p_t).
\end{align*}
For all $t \geq 0$, we have $d_t \in [0, 1/2]$. Thus $d_t \leq 2 p_t (1-p_t)$ and 
\begin{align*}
E[d_t] \leq 2 E[p_t(1-p_t)].
\end{align*}
Using the previous lemma, because a martingale preserves the expectancy, we have
\begin{align*}
E \left[ \frac{p_t (1 - p_t)}{\left( 1 - \sum_{i=1}^\lambda \gamma_i^2 \right)^t} \right] & = E \left[ \frac{p_0 (1 - p_0)}{\left( 1 - \sum_{i=1}^\lambda \gamma_i^2 \right)^0} \right] \\
& = 1/4.
\end{align*}
Therefore
\begin{align*}
E[d_t] &\leq 2 E[p_t(1-p_t)] \\
& = \frac{1}{2} \left( 1 - \sum_{i=1}^\lambda \gamma_i^2 \right)^t.
\end{align*}
Because we have the event inclusion $\{ d_t < 1/3 \} \subset \{ T_L \leq t \}$ and using a Markov inequality on $T_L$, we obtain
\begin{align*}
P[ T_L \leq t ] & \geq P[ d_t < 1/3 ] \\
& > 1 - 3 E[d_t] \\
& \geq 1 - \frac{3}{2} \left( 1 - \sum_{i=1}^\lambda \gamma_i^2 \right)^t.
\end{align*}
Therefore we have $P[ T_L \geq t + 1 ] \leq \frac{3}{2} \left( 1 - \sum_{i=1}^\lambda \gamma_i^2 \right)^t$. Because $T_L$ takes integer values, we can obtain an upper bound on its expectancy
\begin{align*}
E[T_L] & = \sum_{t=1}^{+ \infty} P[ T_L \geq t ] \\
& \leq \sum_{t=0}^{+ \infty} \frac{3}{2} \left( 1 - \sum_{i=1}^\lambda \gamma_i^2 \right)^t \\
& = \frac{3}{2} \frac{1}{1 - \left( 1 - \sum_{i=1}^\lambda \gamma_i^2 \right)}\\
& = \frac{3}{2 \left(\sum_{i=1}^\lambda \gamma_i^2 \right)} \\
& = \mathcal{O}\left( \frac{1}{\sum_{i=1}^\lambda \gamma_i^2} \right).
\end{align*}
We note that we got this result using simple bounds and theorems. We can also find a slightly stronger result by applying the $\log$ function to the martingale to obtain a supermartingale then using Doob's optional sampling theorem.
\end{proof}

\subsection{Proof of Lemma \ref{lemma:concentration}}

We will use the following Hoeffding-Azuma inequality for maxima \cite[Theorem 3.10 and (41)]{McDiarmid98}. 

\begin{lemma}\label{lemma:hoeffding}
Let $a_1, \dots , a_m \in \R$ and $S_1, \dots , S_{m+1}$ be a martingale with $|S_{k+1} - S_{k}| \leq a_k$ for $k \in [1.. m]$. Then for all $\delta \geq 0$, we have
\begin{align*}
P\left[ \max_{k=1..m+1} \left|S_{k} - S_1\right| \geq \delta \right] \leq 2\exp\left(-\frac{\delta^2}{2 \sum_{i=1}^m a_i^2} \right).
\end{align*}
\end{lemma}
\begin{proof}[of lemma \ref{lemma:concentration}]
$(p_t)$ is a martingale, but the difference between two consecutive values ($|p_{t+1} - p_t|$) is too big to reach our result, instead we will construct a new martingale sequence whose values are closer to each other. For $t \geq 0$ and $k \in [0.. \lambda]$, we set
\begin{align*}
q_{\lambda t + k} = p_t \left( 1 - \sum_{i=1}^k \gamma_i \right) + \sum_{i=1}^k \gamma_i x_{t+1}^i.
\end{align*}
Even though $q_{\lambda (t+1)} = q_{\lambda t + \lambda}$ is defined in two different ways, these two definitions yield the same result. For $t \geq 0$, we have $q_{\lambda t} = p_t$. Also for all $t \geq 0, k \in [1.. \lambda ]$, we have
\begin{align*}
|q_{\lambda t + k} - q_{\lambda t + k - 1} | \leq \gamma_k.
\end{align*}
Moreover, $q_0 = 1/2$ and we saw in the previous proof that $(q_i)$ is a martingale, therefore we can apply lemma \ref{lemma:hoeffding} on $(q_i)$. For all $T \in \N$, $\delta \geq 0$
\begin{align*}
P\left[ \exists t \in [0..T], |p_t - 1/2| \geq \delta \right] & \leq P\left[ \max_{i=0..\lambda t} \left|q_{i} - q_0\right| \geq \delta \right] \\
& \leq 2\exp\left(-\frac{\delta^2 }{2 \sum_{k=1}^T \sum_{i=1}^\lambda \gamma_i^2} \right) \\
& = 2\exp\left(-\frac{\delta^2 }{2 T \sum_{i=1}^\lambda \gamma_i^2} \right).
\end{align*}
\end{proof}

}%end sloppy
\end{document}